\documentclass[10pt]{article}
\usepackage[letterpaper,twoside,outer=1.3in,vmargin=1.3in,]{geometry}

\usepackage{amsmath,amsfonts,amsthm,url,graphicx,subcaption,float,booktabs}

\usepackage{amsmath,amssymb,amsthm,multirow,booktabs,tikz,caption,graphicx,subcaption}
\newtheorem{thm}{Theorem}

\newtheorem{prop}[thm]{Proposition}

\newtheorem{example}{Example}

\newcommand{\G}{\mathcal{G}}

\title{Integer Programming for Causal Structure Learning\\ in the Presence of Latent Variables}
\author{Rui Chen \\ University of Wisconsin-Madison \\[-.2cm]\small (rchen234@wisc.edu) \and 
	Sanjeeb Dash\\ IBM Research \\[-.2cm]\small (sanjeebd@us.ibm.com) \and
	Tian Gao \\ IBM Research \\[-.2cm] \small (tgao@us.ibm.com)}

\begin{document}
	\maketitle
	\begin{abstract}
		The problem of finding an ancestral acyclic directed mixed graph (ADMG) that represents the   causal relationships between a set of variables is an important area of research on causal inference. Most existing score-based structure learning methods focus on learning directed acyclic graph (DAG) models without latent variables. A number of  score-based methods have recently been proposed for the ADMG learning, yet they are heuristic in nature and do not guarantee an optimal solution. We propose a novel exact score-based method that solves an integer programming (IP) formulation and returns a score-maximizing ancestral ADMG for a set of continuous variables that follow a multivariate Gaussian distribution. We generalize the state-of-the-art IP model for DAG learning problems and derive new classes of valid inequalities to formulate an IP model for ADMG learning. Empirically, our model can be solved efficiently for medium-sized problems and achieves better accuracy than state-of-the-art score-based methods as well as benchmark constraint-based methods. 
	\end{abstract}
	
	\section{Introduction}
	Causal graphs are graphical models representing dependencies and causal relationships between a set of variables. One class of the most common causal graphs, often known as Bayesian networks (BNs), is modeled by directed acyclic graphs (DAGs) in which a direct causal relationship between two nodes is indicated by a directed edge.  However, its structure learning problem is NP-hard~\cite{chickering2004large}. Many exact and approximate algorithms for learning BN structures from data have been developed, including score-based and constraint-based approaches. Score-based methods use a score -- such as Bayesian scores or Bayesian information criterion (BIC) -- to measure the goodness of fit of different graphs over the data, and then use a search procedure -- such as hill-climbing \cite{heckerman1995learning,MMPCcor,HCimap}, forward-backward search~\cite{chickering02}, dynamic programming \cite{smDP,silander06,gao2018parallel}, A* \cite{Yuan13learning} or integer programming \cite{jaakkola2010learning,cussens11,cussens2016polyhedral} -- in order to find the best graph. On the other hand, constraint-based structure learning algorithms use  (conditional) independence tests to decide on the existence of edges between all pairs of variables.  Popular algorithms include the SGS algorithm \cite{spirtes2000causation},  PC algorithm \cite{spirtes2000causation}, and IC algorithm \cite{pearl2000causality}. 
	
	
	Despite their wide application \cite{pearl2000causality}, it is known that DAG models are not closed under marginalization \cite{tsirlis2018scoring}. This implies that DAGs cannot be used to model structures involving latent variables which are the most common cases in practice.
	For example, in healthcare domains, there may be numerous unobserved factors such as gene expression. Self-reported family history and diets may also leave out some important information. Ancestral graphs (AGs) were proposed as a generalization of DAGs \cite{richardson2002ancestral}. AG models include all DAG models and are closed under marginalization and conditioning. AGs capture the independence relations among observed variables without explicitly including latent variables in the model. In this work, we assume no selection biases (no undirected edges in AGs). Hence there are two types of edges in the AGs: {\em directed} edges represent direct or ancestral causal relationships between variables and {\em bidirected} edges represent latent confounded relationships between variables. A bidirected edge between two nodes means that there exists at least one latent confounder that causes both nodes, and neither node causes the other node. AGs considered in this work  are also called ancestral acyclic directed mixed graphs (ADMGs) \cite{bhattacharya2020differentiable} in the  literature.
	
	
	Causal structure learning algorithms for ancestral ADMGs can also be divided into two main classes: constraint-based  and score-based methods. Constraint-based methods apply conditional independence tests on the data to infer  graph structures. Score-based methods search through possible graph structures to optimize a criterion for model selection. There are also hybrid methods that use both conditional independent tests and some scoring criterion that measures the likelihood of the data. In the setting of learning DAGs from observational data, score-based methods often achieve better performance than constraint-based methods as score-based methods are less sensitive to error propagation \cite{spirtes2010introduction}. For ADMG learning, existing methods are mostly constraint-based, including the FCI algorithm \cite{spirtes2000causation,zhang2008completeness} and the conservative FCI (cFCI) \cite{ramsey2012adjacency} to name a couple. Several score-based or hybrid approaches have been proposed for ancestral ADMG structure learning over continuous Gaussian variables in recent years \cite{triantafillou2016score,tsirlis2018scoring,bernstein2020ordering,chobtham2020bayesian} but are all greedy or local search algorithms. In this paper, we close the gap and propose an exact score-based solution method based on an integer programming (IP) formulation for ancestral ADMG learning. Our method is inspired by existing DAG learning IP formulations, and we derive new classes of valid inequalities to restrict the learned graph to be a valid ADMG. Empirical evaluation shows that the proposed method outperforms existing score-based DAG and ADMG  structure learning algorithms. 
	
	The paper is organized as follows. In Section \ref{sec:ADMG_learning}, we define the ancestral ADMG learning problem. In Section \ref{sec:IP_formulation}, we propose an integer programming formulation for obtaining the score-maximizing ancestral ADMG. In Section \ref{sec:numerical}, we present experiments to compare our method with existing baselines.

	\section{Ancestral ADMG Learning}\label{sec:ADMG_learning}
	
	\subsection{Preliminaries}
	We briefly review some related concepts. DAGs are directed graphs without directed cycles.
	A directed mixed graph $\G = (V,E_d, E_b)$ consists of a set of nodes $V$, a set of directed edges ($\rightarrow$) $E_d \subseteq \{(i,j) : i, j \in V, i \neq j\}$, and bidirected edges ($\leftrightarrow$) $E_b \subseteq \{\{i,j\} : i, j \in V, i \neq j\}$ between certain pairs of nodes. Given a directed edge $a \rightarrow b$,  $b$ is the head node, and $a$ the tail node. We call node $a$ an ancestor of node $b$ in $\G$ if there is a directed path from $a$ to $b$ in $\G$ or $a=b$. We call node $a$ a spouse (resp., parent) of node $b$ in $\G$ if there is a bidirected edge between $a$ and $b$ (resp., a directed edge from $a$ to $b$) in $\G$. We denote the set of ancestors, the set of parents, and the set of spouses of node $a$ in $\G$ by an$_\G(a)$, pa$_\G(a)$, and sp$_\G(a)$, respectively. For $i\in V$ and $W\subseteq V$, $W\rightarrow i$ denotes that $W$ is the parent set of $i$. A directed mixed graph $\G$ is called an {\em ancestral ADMG} if the following condition holds for all pairs of nodes $a$ and $b$ in $\G$:\begin{itemize}
		\item If $a\neq b$ and $b\in$an$_\G(a)\cup$sp$_\G(a)$, then $a\notin$an$_\G(b)$.
	\end{itemize}
	In other words, $\G$ is an ancestral ADMG if it contains no directed cycles ($a\rightarrow c\rightarrow\ldots\rightarrow b\rightarrow a$) or almost directed cycles. An almost directed cycle is of the form $a\rightarrow c\rightarrow\ldots\rightarrow b\leftrightarrow a$; in other words, $\{a,b\} \in E_b$ is a bidirected edge, and $a \in$ an$_\G(b)$. Given a directed mixed graph $\G$, the \textit{districts} define a set of equivalence classes of nodes in $\G$. The district for node $a$ is defined as the  connected component of $a$ in the subgraph of $\G$ induced by all bidirected edges, i.e.,\begin{displaymath}
	\{b:b\leftrightarrow\ldots\leftrightarrow a\text{ in }\G\text{ or }a=b\}.
	\end{displaymath}
	Given a district $D$ of $\G$, the directed mixed graph $\G_D$ is a subgraph of $G$ defined as follows. The node set of $\G_D$ consists of all nodes in $D$ and their parents. The bidirected edges in $\G_D$ are the bidirected edges in $\G$ that connect nodes in $D$. The directed edges in $\G_D$ are the directed edges in $\G$ where the head nodes are contained in $D$. We say that $\G_D$ is the subgraph {\em implied} by $D$. In this paper, we call subgraph $C$ of $\G$ a c-component if $C$ is a subgraph implied by some district of $\G$. Note that if a c-component has no bidirected edges, then any district must consist of a single node, and the c-component consists of a number of directed edges all with the same head node. Some authors use c-component as a synonym for district.

	\subsection{Calculating Graph Scores}\label{subsec:scoring}
	We assume the ground truth data is modeled by the following linear equations:\begin{displaymath}
	X=MX+\epsilon.
	\end{displaymath}
	Here $M$ is a $d\times d$ matrix, $X=(X_1,\ldots,X_d)^T$ are variables of the model, and $\epsilon=(\epsilon_1,\ldots,\epsilon_d)^T$ is the error vector which follows a multivariate Gaussian distribution $N(0,\Sigma)$\footnote{We focus on the multivariate Gaussian case in this work as the scoring function can be factorized, although our work may also apply to discrete cases per different factorization rules \cite{richardson2014factorization,evans2014markovian}.}. Neither $M$ nor $\Sigma$ can be observed. Our goal is to find the best-fitting ancestral ADMG $\G$ and an associated parameterization $(M,\Sigma)$ satisfying $M_{ij}=0$ if $i=j$ or $j\rightarrow i$ is not in $\G$ and $\Sigma_{ij}=0$ if $i\leftrightarrow j$ is not in $\G$. A score of the graph measures how well the graph structure represents the data. We use BIC \cite{schwarz1978estimating} scores for all graphs in this paper. The BIC score for graph $\G$ is given by\begin{displaymath}
	\textbf{BIC}_\G=2\ln(L_\G(\hat{\Sigma}))-\ln(N)(2|V|+|E|).
	\end{displaymath}
	Here $\hat{\Sigma}$ is the maximum likelihood estimate of $\Sigma$ given the graph representation $\G$ and $\ln(L_\G(\hat{\Sigma}))$ is the associated log-likelihood while $N$ denotes the number of samples, $|V|$ and $|E|$ denote the number of nodes and the number of edges in $\G$, respectively. Given a fixed graph and the empirical covariance matrix $Q$, the maximum likelihood estimate of the parameters can be found by applying the residual iterative conditional fitting algorithm in \cite{drton2009computing}. According to \cite{nowzohour2017distributional}, $\ln(L_\G(\hat{\Sigma}))$ can be decomposed by c-components in $\G$. Specifically, let $\mathcal{D}$ denote all districts of $\G$. Then\begin{multline}
	\ln(L_\G(\hat{\Sigma}))=\\
	-\frac{N}{2}\sum_{D\in\mathcal{D}}\Big[|D|\ln(2\pi)+
	\log(\frac{|\hat{\Sigma}_{\G_D}|}{\prod_{j\in\text{pa}_{\G}(D)}\hat{\sigma}^2_{Dj}})+\\
	\frac{N-1}{N}\text{tr}(\hat{\Sigma}^{-1}_{\G_D}Q_D-|\text{pa}_\G(D)\setminus D|)\Big].    
	\end{multline}
	Here pa$_{\G}(D)$ denotes the union of parent sets of nodes in $D$, $\hat{\Sigma}_{\G_D}$ is the maximum log-likelihood for subgraph $\G_D$ and $\hat{\sigma}^2_{Dj}$ denotes diagonal entry of $\hat{\Sigma}_{\G_D}$ corresponding to node $j$. Note that the districts partition the nodes of $\G$, and the edges of $\G$ are partitioned by the subgraphs $\G_D$. The BIC score for graph $\mathcal{G}$ can be expressed as a sum of local scores of its c-components. For example, in Figure \ref{fig:decomp}, the BIC score of the ancestral ADMG is equal to the sum of local scores of four c-components represented by different colors. Nodes with the same color belong to the same district and directed edges indicate their parents. When the graph is a DAG, other decomposable scoring functions can also be used \cite{silander2008factorized}.
	\begin{figure}[bt!]
		\centering
		\begin{tikzpicture}
		\filldraw [red] (0,0) circle (2pt);
		\filldraw [red] (0,2) circle (2pt);
		\filldraw [blue] (2,0) circle (2pt);
		\filldraw [red] (2,2) circle (2pt);
		\filldraw [yellow] (-1,1) circle (2pt);
		\filldraw [blue] (3,1) circle (2pt);
		\filldraw [green] (4,1) circle (2pt);
		\draw[->,red] (-0.9,1.1) -- (-0.1,1.9);
		\draw[<->,red] (0,0.1) -- (0,1.9);
		\draw[<->,red] (0.1,2) -- (1.9,2);
		\draw[<-,red] (0.1,0) -- (1.9,0);
		\draw[<->,blue] (2.1,0.1) -- (2.9,0.9);
		\draw[->,blue] (2.1,1.9) -- (2.9,1.1);
		\draw[->,green] (3.1,1) -- (3.9,1);
		\draw[->,green] (2.1,2) -- (3.9,1.1);
		\draw[->,green] (2.1,0) -- (3.9,0.9);
		\end{tikzpicture}
		\caption{Decomposition of the BIC score for an ancestral ADMG.}
		\label{fig:decomp}
	\end{figure}
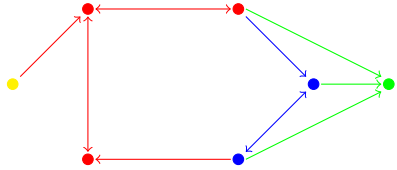
	
	Given the above decomposition property of BIC scores, one can immediately formulate the ancestral ADMG learning problem as an optimization problem, where one takes as input the set $\mathcal{C}$ of all possible candidate c-components defined on the input random variables, and considers each subset of $\mathcal{C}$ where the union of edges in the c-components in the subset form an ancestral ADMG.
	This is however impractical as the size of $\mathcal{C}$  grows exponentially with increasing number of nodes/random variables. In a similar fashion, the DAG learning problem can be solved by considering candidate c-components where each c-components is implied by a single node district. Even with this restriction, the number of c-components that need to be considered in DAG learning is exponential, and a common approach to deal with this issue is to only consider c-components with a bounded number of nodes.
	The model we propose for ADMG learning is defined over general c-components, although in our computational experiments we restrict the set of candidate c-components to allow our model to be solved within a reasonable amount of time. In Section \ref{subsec:prune}, we further consider eliminating sub-optimal c-components, which significantly reduces the size of the search space in some cases.
	\subsection{Pruning the list of candidate c-components}\label{subsec:prune}
	The set of ancestral ADMGs defined over a set of nodes is closed under deletion of directed and bidirected edges.
	Similarly, the operation of removing all parents of a node and replacing all bidirected edges incident to the node with outgoing directed edges transforms an ancestral ADMG into another ancestral ADMG.
	Both of these operations also result in the transformation of the c-components associated with an ADMG.
	If the c-components resulting from applying the above transformation to another c-component have a higher combined score than the original c-component, then we can assert that the original c-component need never be considered in an optimal solution of the ADMG learning problem.
	We use the above ideas and their hybrids to define two ways of pruning $\mathcal{C}$ in our implementation. Firstly, if removing one or more edges from a c-component results in a set of c-components with a higher combined score, we prune the original c-component. For example, 
	the c-component $\{A\}\rightarrow B\leftrightarrow C\leftarrow \{D\}$ can be pruned if it has a lower score than the sum of the scores of $B\leftarrow \{A\}$ and $C\leftarrow \{D\}$. Similarly, we prune a c-component if the second operation above leads to c-components with a higher combined score. For example, we can prune  $\{A\}\rightarrow B\leftrightarrow C\leftarrow \{D\}$ if it has a score lower than the sum of the scores of $B\leftarrow\emptyset$ and $C\leftarrow\{B,D\}$. A hybrid of the previous two pruning ideas is also implemented. For example, we can combine both operations to transform $\{A\}\rightarrow B\leftrightarrow C\leftarrow \{D\}$ into the c-components $B\leftarrow\emptyset$ and $C\leftarrow\{B\}$.
	
	\section{An Integer Programming Formulation}\label{sec:IP_formulation}
	Integer programming (IP) is a mathematical optimization tool for modeling and solving optimization problems involving variables that are restricted to discrete values. For any c-component $C$ implied by a district, let $D_C$ denote the implying district, let $E_{C}$ denote the pairs of nodes in $D_C$ that are connected by a bidirected edge, and for each node $i \in D_C$, let $W_{C,i}$ denote the parent set of node $i$ in c-component $C$. Given a set of c-components $\mathcal{C}$, finding a score-maximizing ancestral ADMG with all its district-implied subgraphs in $\mathcal{C}$ can be straightforwardly formulated as an integer program as follows:\begin{eqnarray}
	\max_{z\in\{0,1\}^\mathcal{C}}\  &\sum_{C\in\mathcal{C}}s_{C}z_{C}\label{obj}\\
	\text{s.t. } &\sum_{C:i\in D_C}z_{C}=1, i\in \{1,\ldots,d\}\label{sumTo1}\\
	&\G(z)\text{ is acyclic and ancestral.}\label{nocycon}
	\end{eqnarray}
	Here $z_C$ is a binary variable indicating if the c-component $C$ is chosen, $s_C$ is the local score of c-component $C$ and $\G(z)$ is the directed mixed graph whose district-implied subgraphs are exactly the c-components $C$ with $z_C=1$. Constraints (\ref{sumTo1}) enforce the condition each node $i$ is contained in a single district (that implies a chosen c-component). Constraint (\ref{nocycon}) implies that the resulting directed mixed graph $\G(z)$ contains no directed or almost directed cycles. Standard IP solvers can only deal with linear inequality constraints, and we next discuss how to represent constraint (\ref{nocycon}) by a set of linear inequality constraints. 
	\subsection{Avoiding Directed Cycles}
	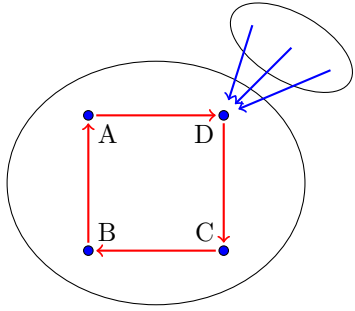
\begin{figure}[bt!]
		\centering
		\begin{tikzpicture}[scale=0.9]
		\draw [fill=blue] (0,0) circle (2pt) node[below right] {A};
		\draw [fill=blue] (0,-2) circle (2pt) node[above right] {B};
		\draw [fill=blue] (2,-2) circle (2pt) node[above left] {C};
		\draw [fill=blue] (2,0) circle (2pt) node[below left] {D};
		\draw (1,-1) ellipse (2.2cm and 1.8cm);
		\draw[rotate around={-30:(3,1)}] (3,1) ellipse (1cm and 0.5cm);
		\draw[red,->,thick,shorten <= 3pt, shorten >= 3pt] (0,0) -- (2,0);
		\draw[red,->,thick,shorten <= 3pt, shorten >= 3pt] (2,0) -- (2,-2);
		\draw[red,->,thick,shorten <= 3pt, shorten >= 3pt] (2,-2) -- (0,-2);
		\draw[red,->,thick,shorten <= 3pt, shorten >= 3pt] (0,-2) -- (0,0);
		\draw[blue,->,thick,shorten >= 6pt] (2.423,1.333) -- (2,0);
		\draw[blue,->,thick,shorten >= 6pt] (3,1) -- (2,0);
		\draw[blue,->,thick,shorten >= 6pt] (3.577,0.667) -- (2,0);
		\end{tikzpicture}
		\caption{Depiction of the cluster inequality for $S=\{A, B, C, D\}$.}
		\label{fig:dircyc}
	\end{figure}
	Several (mixed-)integer programming (MIP) formulations exist for constraining a graph to be acyclic. A description and a comparison of some of these different anti-cycle formulations for the DAG learning problem with continuous data, including the linear ordering formulation \cite{grotschel1985acyclic}, the topological ordering formulation \cite{park2017bayesian}, and the cycle elimination formulation, are given in \cite{manzour2020integer}. A new layered network formulation was also proposed in \cite{manzour2020integer}. One particular class of anti-cycle constraints called cluster inequalities/constraints was introduced in \cite{jaakkola2010learning} and often performs better than the other methods above and results in a very tight integer programming formulation for DAG learning problems. Let $I(\cdot)$ be a variable indicating whether or not the substructure described between the parentheses is present in the graph. Specifically, $I(W\rightarrow i)$ indicates whether or not pa$_{\G(z)}(i)=W$.
	Cluster constraints have the following form:\begin{equation}\label{ineq:cluster}
	\sum_{i\in S}\sum_{W:W\cap S=\emptyset}I(W\rightarrow i)\geq 1,~~\forall S\subseteq\{1,2,\ldots,d\}.
	\end{equation}
	Inequalities \eqref{ineq:cluster} encode the constraint that for a set $S$ of nodes in an acyclic graph, there must exist at least one node in $S$ whose parent set has no intersection with $S$. In Figure~\ref{fig:dircyc}, there is a directed cycle connecting nodes $\{A, B, C, D\}$, and it violates the cluster inequality for $S = \{A, B, C, D\}$, which would be satisfied if node $D$ has all its parents outside $S$. The generation of cluster constraints is essential for the state-of-the-art IP-based DAG learning solver GOBNILP \cite{bartlett2017integer}. The same inequalities can be applied 
	to ancestral ADMG learning by introducing auxiliary variables $I(W\rightarrow i)$ defined as\begin{equation}\label{def:aux_var}
	I(W\rightarrow i):=\sum_{C\in\mathcal{C}:i\in D_C,W_{C,i}=W}z_C.
	\end{equation}
	In our case, inequalities (\ref{ineq:cluster}) can be further strengthened as some of the original variables $z_C$ may be double counted when they are present in multiple auxiliary variables of the form $I(W\rightarrow i)$. We propose the following {\em strengthened cluster constraints} for ancestral ADMG learning:\begin{equation}\label{ineq:strengthened_cluster}
	\sum_{i\in S}\sum_{C\in \mathcal{C}:i\in D_C,W_{C,i}\cap S=\emptyset}z_C\geq 1,~~\forall S\subseteq\{1,2,\ldots,d\}.
	\end{equation}
	\begin{prop}\label{prop:ADMG}
		Inequalities \eqref{ineq:strengthened_cluster} are satisfied by all solutions of the integer program \eqref{obj}-\eqref{nocycon}. If $z\in\{0,1\}^d$ satisfies \eqref{sumTo1} and \eqref{ineq:strengthened_cluster}, then $z$ corresponds to an ADMG, i.e., $\G(z)$ contains no directed cycles.
	\end{prop}
	\begin{proof}
		We first show the validity of \eqref{ineq:strengthened_cluster}. Let $\bar{z}$ be a feasible solution of the integer program \eqref{obj}-\eqref{nocycon}. Assume that inequality \eqref{ineq:strengthened_cluster} with $S$ set to $\bar{S}$ is violated by $\bar{z}$, i.e.,\begin{equation}\label{viol}
		\sum_{i\in \bar{S}}\sum_{C\in \mathcal{C}:i\in D_C,W_{C,i}\cap \bar{S}=\emptyset}\bar{z}_C=0.
		\end{equation}
		By constraint $\eqref{sumTo1}$, for each $i\in\{1,\ldots,d\}$, there exists $C_i$ satisfying $i\in D_{C_i}$ such that $\bar{z}_{C_i}=1$. By \eqref{viol}, we have $W_{C_i,i}\cap\bar{S}\neq \emptyset$ for each $i\in\bar{S}$. It follows that each node $i$ in $\bar{S}$ has a parent in $\bar{S}$ in $\G(\bar{z})$, which contradicts the feasibility of $\bar{z}$ since $\G(\bar{z})$ contains a directed cycle with nodes in $\bar{S}$.
		
		Next we show that $\G(z)$ is an ADMG if $z$ satisfies \eqref{sumTo1} and \eqref{ineq:strengthened_cluster}. Constraints \eqref{sumTo1} imply that $\G(z)$ is a valid directed mixed graph as exactly one c-component having $i$ in its implying district is active. We only need to show that $\G(z)$ contains no directed cycles. Assume for contradiction that $\G(z)$ contains a directed cycle $i_0\rightarrow i_1\rightarrow\ldots\rightarrow i_K\rightarrow i_0$. Let $S'=\{i_0,i_1,\ldots,i_K\}$. By constraint $\eqref{sumTo1}$, for each $i\in\{1,\ldots,d\}$, there exists exactly one $C_i$ satisfying $i\in D_{C_i}$ such that $z_{C_i}=1$. The directed cycle $i_0\rightarrow i_1\rightarrow\ldots\rightarrow i_K\rightarrow i_0$ implies that $W_{C_i,i}\cap S'\neq\emptyset$ for each $i\in S'$. Therefore, by constraint $\eqref{sumTo1}$, $\sum_{C\in \mathcal{C}:i\in D_C,W_{C,i}\cap S'=\emptyset}z_C=0$ for each $i\in S'$, which contradicts inequality \eqref{ineq:strengthened_cluster} with $S=S'$.
	\end{proof}
	\subsection{Avoiding Almost Directed Cycles}
	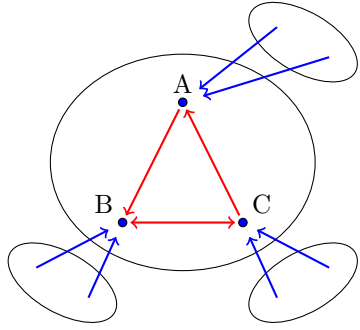
\begin{figure}[bt!]
		\centering
		\begin{tikzpicture}[scale=0.8]
		\draw [fill=blue] (1,0) circle (2pt) node[above] {A};
		\draw [fill=blue] (0,-2) circle (2pt) node[above left] {B};
		\draw [fill=blue] (2,-2) circle (2pt) node[above right] {C};
		\draw [red,<->,thick,shorten <= 3pt, shorten >= 3pt] (0,-2) -- (2,-2);
		\draw [red,->,thick,shorten <= 3pt, shorten >= 3pt] (1,0) -- (0,-2);
		\draw [red,->,thick,shorten <= 3pt, shorten >= 3pt] (2,-2) -- (1,0);
		\draw (1,-1) ellipse (2.2cm and 1.8cm);
		\draw[rotate around={-30:(3,1)}] (3,1) ellipse (1cm and 0.5cm);
		\draw[blue,->,thick,shorten >= 8pt] (2.567,1.25) -- (1,0);
		\draw[blue,->,thick,shorten >= 8pt] (3.433,0.75) -- (1,0);
		\draw[rotate around={30:(3,-3)}] (3,-3) ellipse (1cm and 0.5cm);
		\draw[rotate around={-30:(-1,-3)}] (-1,-3) ellipse (1cm and 0.5cm);
		\draw[blue,->,thick,shorten >= 6pt] (3.433,-2.75) -- (2,-2);
		\draw[blue,->,thick,shorten >= 6pt] (2.567,-3.25) -- (2,-2);
		\draw[blue,->,thick,shorten >= 6pt] (-1.433,-2.75) -- (0,-2);
		\draw[blue,->,thick,shorten >= 6pt] (-0.567,-3.25) -- (0,-2);
		\end{tikzpicture}
		\caption{Depiction of the bicluster inequality for $S=\{A, B, C\}$ and $(i,j)=(B,C)$.}
		\label{fig:almostdircyc}
	\end{figure}
	Ancestral ADMGs are ADMGs without almost directed cycles. To encode the constraint that the graph contains no almost directed cycles, we define auxiliary variables $I(W^1\rightarrow i,W^2\rightarrow j,i\leftrightarrow j)$ and $I(i\leftrightarrow j)$ in a manner similar to \eqref{def:aux_var},
	and give the following inequalities for all $S\subseteq \{1,2,\ldots,d\}$ with $i, j\in S$ and $i< j$:\begin{multline}\label{ineq:bicluster}
	\sum_{W^1:W^1\cap S=\emptyset}\sum_{W^2:W^2\cap S=\emptyset}I(W^1\rightarrow i,W^2\rightarrow j,i\leftrightarrow j)+\\
	\sum_{v\in S\setminus\{i,j\}}\sum_{W:W\cap S=\emptyset}I(W\rightarrow v)\geq I(i\leftrightarrow j).
	\end{multline}
	We refer to inequalities \eqref{ineq:bicluster} as the bicluster inequalities/constraints. Bicluster inequalities encode the constraint that if a bidirected edge $i\leftrightarrow j$ is present in some ancestral ADMG, then, for any set of nodes $S$ containing $i$ and $j$, either some node $v\in S\setminus\{i,j\}$ has all its parents outside of $S$ or both parent sets of $i$ and $j$ have no intersection with $S$. As shown in Figure~\ref{fig:almostdircyc}, where $B$ and $C$ are connected by a bidirected edge,  either the parents of $A$ must lie outside $S=\{A, B, C\}$, or the parents of both $B$ and $C$ must lie outside $S$. Similar to cluster inequalities, bicluster inequalities can also be strengthened when written in the original $z$ variable space as some c-component variables are double-counted on the left-hand side of \eqref{ineq:bicluster}. Also some c-component variables might be contradicting the presence of the bidirected edge $i\leftrightarrow j$, and therefore cannot be active when $I(i\leftrightarrow j)$ is active and should be removed from the left-hand side.\begin{prop}
		The following inequalities are valid for the integer program \eqref{obj}-\eqref{nocycon}:\begin{multline}\label{ineq:lift_bicluster}
		\sum_{C\in\mathcal{C}(S;\{i,j\})}z_C\geq I(i\leftrightarrow j), ~~\forall S\subseteq\{1,2,\ldots,d\}\\
		\text{ with }i\in S,j\in S,i<j
		\end{multline}
		where \begin{multline*}
		\mathcal{C}(S;\{i,j\})=\\
		\{C\in\mathcal{C}:\{i,j\}\in E_C,(W_{C,i}\cup W_{C,j})\cap S=\emptyset \}\cup\\
		\{C\in \mathcal{C}:|\{i,j\}\cap D_C|\neq 1,\exists v\in D_C\cap S\setminus\{i,j\}\\
		\text{ s.t. }W_{C,v}\cap S=\emptyset\}.
		\end{multline*}
		Any $z\in\{0,1\}^d$ satisfying \eqref{sumTo1}, \eqref{ineq:strengthened_cluster} and \eqref{ineq:lift_bicluster} corresponds to an ancestral ADMG.
	\end{prop}
	\begin{proof}
		We first show validity of \eqref{ineq:lift_bicluster}. Let $\bar{z}$ be a feasible solution of \eqref{obj}-\eqref{nocycon}. Assume for contradiction that inequality \eqref{ineq:lift_bicluster} with $S=\bar{S}$ and $(i,j)=(\bar{i},\bar{j})$ is violated by $\bar{z}$. Note that the right-hand side value of \eqref{ineq:lift_bicluster} can only be binary and \eqref{ineq:lift_bicluster} cannot be violated if the right hand side is $0$ since $z$ is binary. Then there exists a bidirected edge between $\bar{i}$ and $\bar{j}$ in $\G(\bar{z})$. Since $z=\bar{z}$ violates \eqref{ineq:lift_bicluster} with $S=\bar{S}$ and $(i,j)=(\bar{i},\bar{j})$, $\bar{z}_C=0$ for each $C\in\mathcal{C}(\bar{S};\{\bar{i},\bar{j}\})$. Next we construct an almost directed cycle in $\G(\bar{z})$. Because variable $I(\bar{i}\leftrightarrow \bar{j})$ is active, there exists $C_{\bar{i}\bar{j}}\in\mathcal{C}$ such that $\{\bar{i},\bar{j}\}\in E_{C_{\bar{i}\bar{j}}}$ and $\bar{z}_{C_{\bar{i}\bar{j}}}=1$. Since $\bar{z}_C=0$ for each $C\in\mathcal{C}(\bar{S};\{\bar{i},\bar{j}\})$, we have $(W_{C_{\bar{i}\bar{j}},\bar{i}}\cup W_{C_{\bar{i}\bar{j}},\bar{j}})\cap \bar{S}\neq \emptyset$. Without loss of generality, assume $W_{C_{\bar{i}\bar{j}},\bar{i}}\cap \bar{S}\neq \emptyset$. Let $i_0=\bar{i}$. We recursively define $i_k$ as follows:\begin{itemize}
			\item Let $i_k$ be a parent of $i_{k-1}$ such that $i_k\in \bar{S}$.
		\end{itemize} 
		Then this sequence either finds a directed cycle because of the finiteness of the number of nodes, which contradicts the feasibility of $\bar{z}$, or ends up with some node $i_K$ that has no parent in $\bar{S}$. If it is the second case, we next show that $i_K$ can only be $j$ which implies an almost directed cycle $i_K\rightarrow i_{K-1}\rightarrow\ldots\rightarrow i_0\leftrightarrow i_K$ in $\G(\bar{z})$. First of all, $i_0=\bar{i}$ has a parent in $\bar{S}$ since $W_{C_{\bar{i}\bar{j}},\bar{i}}\cap \bar{S}\neq \emptyset$. Let $\bar{C}\in\mathcal{C}$ be such that $i_K\in D_{\bar{C}}$ and $z_{\bar{C}}=1$. Assume for contradiction that $i_K\neq j$. Then $i_K\in D_{\bar{C}}\cap S\setminus\{i,j\}$ and $W_{\bar{C},i_K}\cap \bar{S}=\emptyset$. Also note that either $\{i,j\}\cap D_{\bar{C}}=\emptyset$ or $\{i,j\}\subseteq D_{\bar{C}}$ since otherwise it contradicts the activeness of $I(i\leftrightarrow j)$. This implies that $\bar{C}\in \mathcal{C}(\bar{S};\{\bar{i},\bar{j}\})$ which contradicts the fact that $\bar{z}_C=0$ for each $C\in\mathcal{C}(\bar{S};\{\bar{i},\bar{j}\})$.
		
		We next show that $\G(z)$ is an ancestral ADMG if $z$ satisfies \eqref{sumTo1}, \eqref{ineq:strengthened_cluster} and \eqref{ineq:lift_bicluster}. By Proposition \ref{prop:ADMG}, $\G(z)$ is an ADMG. We only need to show that $\G(z)$ contains no almost directed cycles. Assume for contradiction that $\G(z)$ contains an almost directed cycle $i_0\rightarrow i_1\rightarrow\ldots\rightarrow i_K\leftrightarrow i_0$. Let $S'=\{i_0,i_1,\ldots,i_K\}$. We can then easily verify that $z_{C}=0$ for any $C\in \mathcal{C}(S';\{i_0,i_K\})$ which contradicts \eqref{ineq:lift_bicluster} with $S=S'$, $i=\min\{i_0,i_K\}$ and $j=\max\{i_0,i_K\}$.
	\end{proof}
	We provide a simple example of a strengthened bicluster inequality when the problem contains only 3 variables.
	\begin{example}
		Let $d=3$ and $\mathcal{C}$ be the collection of all possible c-components on three nodes. Consider $S=\{1,2,3\}$, $i=2$ and $j=3$. Then the asociated strengthened bicluster inequality (after simplification) is\begin{equation}\label{exmp:bicluster}
		z_{\emptyset\rightarrow1}\geq z_{\{1\}\rightarrow2\leftrightarrow3\leftarrow\{1\}}+z_{\emptyset\rightarrow2\leftrightarrow3\leftarrow\{1\}}+z_{\{1\}\rightarrow2\leftrightarrow3\leftarrow\emptyset}.
		\end{equation}
		Note that at most one of the c-components on the right-hand side of \eqref{exmp:bicluster} ($\{1\}\rightarrow2\leftrightarrow3\leftarrow\{1\}$, $\emptyset\rightarrow2\leftrightarrow3\leftarrow\{1\}$ and $\{1\}\rightarrow2\leftrightarrow3\leftarrow\emptyset$) can be active due to \eqref{sumTo1}. Inequality \eqref{exmp:bicluster} enforces the constraint that c-component $\emptyset\rightarrow1$ must be active when one of the three c-components on the right-hand side of \eqref{exmp:bicluster} is active.
	\end{example}
	
	So far we have given an integer programming formulation of the ancestral ADMG learning problem. However, solving this integer program efficiently is nontrivial. The first step to solve an integer program is often solving some polyhedral relaxation of the problem. Directly solving the relaxed problem over the polyhedron defined by \eqref{sumTo1}, \eqref{ineq:strengthened_cluster} and \eqref{ineq:lift_bicluster} is computationally infeasible because of the exponential number of constraints. A common practice for solving such problems with exponentially many constraints is to further relax the polyhedral relaxation to be the polyhedron defined by only some of these constraints, find an optimal solution, solve a corresponding {\em separation problem}, and add violated constraints to the relaxation if they exist. The separation problem finds constraints that are violated by the current solution and this process is repeated until convergence. If the optimal solution of the relaxed problem is integral, such a separation problem translates to just finding a directed or almost directed cycle in a directed mixed graph which can easily be accomplished by depth-first search. However, the separation problem can be much harder to solve when the solution is fractional. We provide separation heuristics based on variants of Karger's random contraction algorithm \cite{karger1993global} in the supplement.
	
	\section{Empirical Evaluation}\label{sec:numerical}
	We conduct a set of experiments to compare our IP model, AGIP, with existing state-of-the-art baselines. We use DAGIP to represent existing DAG IP models such as the one in GOBNILP \cite{bartlett2017integer}, which is the same IP model as AGIP when all candidate c-components are  implied by single-node districts. The solution obtained from DAGIP is equivalent to any exact score-based method for generating DAG solutions. We also compare with non-IP-based approaches, namely M$^3$HC \cite{tsirlis2018scoring}, FCI \cite{spirtes2000causation,zhang2008completeness}, and cFCI \cite{ramsey2012adjacency}.
	
	
	To measure solution quality, we use a few different metrics. When comparing against score-based methods, our optimization model objective is to maximize a score, such as the BIC score \cite{schwarz1978estimating} for model selection, and the metric is the solution score. To compare with constraint-based methods which do not have objective scores,  the solution graph is converted to a partial ancestral graph (PAG), which characterizes a class of Markov equivalent AG solutions, and then compared with the ground truth PAG. We use structural Hamming distance (SHD) \cite{tsamardinos2006max}, which is the number of edge operations (addition or deletion of an undirected edge, addition, removal or reversion of the orientation of an edge) between the predicted graph and the ground truth graph. Finally, precision and recall \cite{tillman2011learning} are also compared. They are the number of correct  edges with correct orientations in the predicted graph divided by the number of edges in the predicted graph and by the number of edges in the ground truth graph, respectively.

	

	All experiments are run on a Windows laptop with 16GB RAM and an Intel Core i7-7660U processor running at 2.5GHz. Integer programs are solved using the optimization solver Gurobi 9.0.3.
	
	\subsection{ Experiment 1: Exact Graph Recovery}
	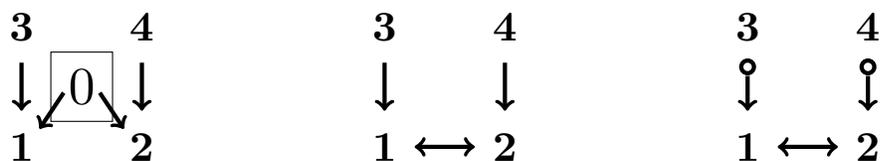
\begin{figure}[hbt!]
		\centering
		\begin{subfigure}{0.3\columnwidth}
			\centering
			\begin{tikzpicture}[scale=0.8]
			\node[scale=1.6] at (0,0) {\textbf{3}};
			\draw[->,ultra thick] (0,-0.6) -- (0,-1.4);
			\node[scale=1.6] at (0,-2) {\textbf{1}};
			\node[draw,scale=2] at (1,-1) {0};
			\draw[->,ultra thick] (0.7,-1.1) -- (0.3,-1.7);
			\node[scale=1.6] at (2,-2) {\textbf{2}};
			\node[scale=1.6] at (2,0) {\textbf{4}};
			\draw[->,ultra thick] (1.3,-1.1) -- (1.7,-1.7);
			\draw[->,ultra thick] (2,-0.6) -- (2,-1.4);
			\end{tikzpicture}
		\end{subfigure}
		\begin{subfigure}{0.3\columnwidth}
			\centering
			\begin{tikzpicture}[scale=0.8]
			\node[scale=1.6] at (0,0) {\textbf{3}};
			\draw[->,ultra thick] (0,-0.6) -- (0,-1.4);
			\node[scale=1.6] at (0,-2) {\textbf{1}};
			\node[scale=1.6] at (2,-2) {\textbf{2}};
			\node[scale=1.6] at (2,0) {\textbf{4}};
			\draw[<->,ultra thick] (0.5,-2) -- (1.5,-2);
			\draw[->,ultra thick] (2,-0.6) -- (2,-1.4);
			\end{tikzpicture}
		\end{subfigure}
		\begin{subfigure}{0.3\columnwidth}
			\centering
			\begin{tikzpicture}[scale=0.8]
			\node[scale=1.6] at (0,0) {\textbf{3}};
			\draw[->,ultra thick,shorten <= 5pt] (0,-0.6) -- (0,-1.4);
			\node[scale=1.6] at (0,-2) {\textbf{1}};
			\node[scale=1.6] at (2,-2) {\textbf{2}};
			\node[scale=1.6] at (2,0) {\textbf{4}};
			\draw[ultra thick] (0,-0.67) circle (3pt);
			\draw[ultra thick] (2,-0.67) circle (3pt);
			\draw[<->,ultra thick] (0.5,-2) -- (1.5,-2);
			\draw[->,ultra thick,shorten <= 5pt] (2,-0.6) -- (2,-1.4);
			\end{tikzpicture}
		\end{subfigure}
		\caption{Ground truth DAG (left), AG (middle), and PAG (right) of the four-node graph.}
		\label{ground_truth:4-node}
	\end{figure}
	
	We first test on a four-node example, where the data (10000 data points) is simulated from a five-node DAG model (see Figure \ref{ground_truth:4-node}) with node $0$ being unobserved. The purpose of this experiment is to show that for small graphs, where we can practically enumerate and use all possible c-component variables, and with enough samples, exact graph recovery is possible with AGIP.
	We test score-based methods AGIP, DAGIP and M$^3$HC on this example. In AGIP, we consider all possible c-components with arbitrary sizes. In DAGIP, we consider all possible c-components implied by single-node districts.
	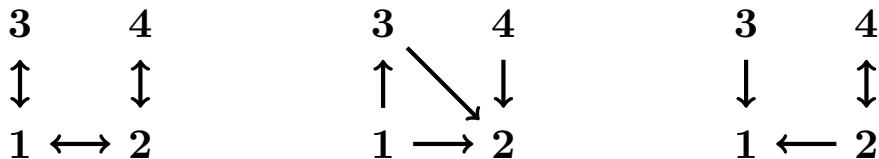
\begin{figure}[hbt!]
		\centering
		\begin{subfigure}{0.3\columnwidth}
			\centering
			\begin{tikzpicture}[scale=0.8]
			\node[scale=1.6] at (0,0) {\textbf{3}};
			\draw[<->,ultra thick] (0,-0.6) -- (0,-1.4);
			\node[scale=1.6] at (0,-2) {\textbf{1}};
			\node[scale=1.6] at (2,-2) {\textbf{2}};
			\node[scale=1.6] at (2,0) {\textbf{4}};
			\draw[<->,ultra thick] (0.5,-2) -- (1.5,-2);
			\draw[<->,ultra thick] (2,-0.6) -- (2,-1.4);
			\end{tikzpicture}
		\end{subfigure}
		\begin{subfigure}{0.3\columnwidth}
			\centering
			\begin{tikzpicture}[scale=0.8]
			\node[scale=1.6] at (0,0) {\textbf{3}};
			\draw[<-,ultra thick] (0,-0.6) -- (0,-1.4);
			\node[scale=1.6] at (0,-2) {\textbf{1}};
			\node[scale=1.6] at (2,-2) {\textbf{2}};
			\node[scale=1.6] at (2,0) {\textbf{4}};
			\draw[->,ultra thick] (0.5,-2) -- (1.5,-2);
			\draw[->,ultra thick] (2,-0.6) -- (2,-1.4);
			\draw[->,ultra thick] (0.4,-0.4) -- (1.6,-1.6);
			\end{tikzpicture}
		\end{subfigure}
		\begin{subfigure}{0.3\columnwidth}
			\centering
			\begin{tikzpicture}[scale=0.8]
			\node[scale=1.6] at (0,0) {\textbf{3}};
			\draw[->,ultra thick] (0,-0.6) -- (0,-1.4);
			\node[scale=1.6] at (0,-2) {\textbf{1}};
			\node[scale=1.6] at (2,-2) {\textbf{2}};
			\node[scale=1.6] at (2,0) {\textbf{4}};
			\draw[<-,ultra thick] (0.5,-2) -- (1.5,-2);
			\draw[<->,ultra thick] (2,-0.6) -- (2,-1.4);
			\end{tikzpicture}
		\end{subfigure}
		\caption{Solutions obtained from AGIP (left), DAGIP (middle) and M$^3$HC (right).}
		\label{solns:4-node}
	\end{figure}
	
	These three methods generate three different solutions. Comparing the scores of the solutions, we observe that score(AGIP) $>$ score(DAGIP) $>$ score(M$^3$HC) in this example. 
	Both AGIP and M$^3$HC correctly capture the skeleton (presence of edges between each pair of nodes) of the ground truth AG. Only the AGIP solution is Markov equivalent to the ground truth AG, i.e., the AGIP solution encodes the same set of conditional independence relationships as the ground truth AG. 
	This result shows that our method is consistent for large samples and finds the exact solution. 
	
	\subsection{Experiment 2: Random Graphs}\label{exp:medium}
	\begin{table}[tb!]
		\centering
		\caption{Tightness of the AGIP formulation.}
		\begin{tabular}{ @{\extracolsep{\fill}} cccccc}
			\toprule
			\multirow{2}{*}{$(d,l,N)$} & Avg \# bin vars & Avg \# bin vars & Avg pruning & Avg root & Avg solution\\
			& before pruning & after pruning &  time (s) &  gap (\%) & time (s)\\
			\midrule\midrule\addlinespace
			$(18,2,\phantom{0}1000)$ & 59229 & 4116 & 19.1 & 0.65 & \phantom{0}60.4\\
			$(16,4,\phantom{0}1000)$ & 39816 & 3590 & 13.6 & 0.43 & \phantom{0}41.0\\
			$(14,6,\phantom{0}1000)$ & 20671 & 1788 & \phantom{0}3.9 & 0.54 & \phantom{00}8.9\\
			$(18,2,10000)$ & 59229 & 9038 & 33.0 & 0.67 & 323.2\\
			$(16,4,10000)$ & 39816 & 7378 & 21.4 & 0.53 & 215.4\\
			$(14,6,10000)$ & 20671 & 3786 & \phantom{0}6.4 & 0.56 & \phantom{0}47.2\\
			\bottomrule
		\end{tabular}
		\label{tab:IP_info}
		\vskip 0.1in
	\end{table}
	
	\begin{table}[tb!]
		\centering
		\caption{Comparing scores of AGIP, DAGIP, and M$^3$HC.}
		\begin{tabular}{ @{\extracolsep{\fill}} cccccc}
			\toprule
			\multirow{2}{*}{$(d,l,N)$} & \multicolumn{2}{c}{Avg improvement in score} & \# instances where AGIP\\
			& \multicolumn{2}{c}{compared with M$^3$HC} & improves over DAGIP in score\\
			\midrule\midrule\addlinespace
			& AGIP & DAGIP\\
			\cmidrule(lr){2-3}
			$(18,2,\phantom{0}1000)$ & \phantom{0}82.75 & \phantom{0}82.32 & 3/10\\
			$(16,4,\phantom{0}1000)$ & \phantom{0}90.03 & \phantom{0}89.33 & 5/10\\
			$(14,6,\phantom{0}1000)$ & \phantom{0}34.84 & \phantom{0}34.68 & 3/10\\
			$(18,2,10000)$ & 373.44 & 373.44 & 0/10\\
			$(16,4,10000)$ & 147.96 & 147.54 & 1/10\\
			$(14,6,10000)$ & 150.52 & 150.44 & 1/10\\
			\bottomrule
		\end{tabular}
		\label{tab:scores}
		\vskip 0.1in
	\end{table}
	We further experiment on a set of randomly generated ADMGs following the procedure described in \cite{triantafillou2016score}. For each instance, a random permutation is applied to $d+l$ variables in a DAG as the ordering of the variables. For each variable $i$, a set of up to 3 variables that have higher ordering than $i$ is randomly chosen as the parent set of $i$. The resulting DAG is then assigned a randomly generated conditional linear Gaussian parameterization. Within the $d+l$ variables in the DAG, $l$ of them are randomly chosen as latent variables and marginalized which result in an AG over the observed $d$ variables. For each such graphical model, a sample of $N\in\{1000,10000\}$ realizations of the observed variables is simulated to create the instance. For fixed $(d,l,N)$, 10 instances with parameters $d$, $l$, and $N$ are generated.
	\subsubsection{Comparison between AGIP, DAGIP and M$^3$HC solutions}\label{exp:score-based}
	To guarantee efficiency, we restrict the sizes of c-components considered in the AGIP and DAGIP methods. AGIP considers c-components implied by a single-node district with up to 3 parents or a two-node district with up to 1 parent per node, while DAGIP considers c-components implied by a single-node district with up to 3 parents.
	
	We want to emphasize that ADMG learning can be much harder than DAG learning. In the state-of-the-art DAG learning IP model \cite{bartlett2017integer}, assuming $n$ variables and parent sets with maximum size $k$,
	there are at least $n\binom{n-1}{k}=\Theta(n^{k+1})$ IP variables in total for fixed $k$ but increasing $n$ (before pruning). For the ADMG learning problem, assuming a maximum district size of $p$ nodes, the number of IP variables in our IP model (AGIP) is at least $\binom{n}{p}\binom{n-p}{k}^p=\Theta(n^{p(k+1)})$ for fixed $k,p$ but increasing $n$ (before pruning). When $p=2$ (the minimum required to model bidirected edges), AGIP has the square of the number of IP variables in DAGIP. With our setting of the experiments, AGIP has roughly double the IP variables of DAGIP. 
	One possible way to deal with this explosive growth is to add a collection of potentially good c-components with large-districts to AGIP rather than all possible ones.

	In these experiments, we fix $d+l=20$ with varying $l\in\{2,4,6\}$. We first show the tightness of our IP formulation. For each combination of $(d,l,N)$, we report in Table \ref{tab:IP_info} the average number of binary variables before and after pruning, average pruning time, average integrality gap at the root node, and average IP solution time for the AGIP model over the 10 instances. We observe that the solution time decreases as $l$ increases. On the other hand, the solution time increases when the number of samples $N$ increases, since fewer c-component variables can be pruned during the pruning phase. For all cases, the IP formulation has a pretty small gap at the root node which is always below 1\% on average.
	
	We next compare the qualities of solutions obtained from AGIP, DAGIP, and M$^3$HC in terms of the scores. Since the feasible region of the AGIP model is strictly larger than DAGIP, the optimal score obtained from AGIP is guaranteed to be at least as good as the optimal score obtained from DAGIP. In Table \ref{tab:scores}, we report the average difference (improvement) in score compared with M$^3$HC and the number of instances where AGIP improves the score over DAGIP for each combination of $(d,l,N)$. Both AGIP and DAGIP produce solutions better than the greedy algorithm M$^3$HC, although M$^3$HC in principle searches over a larger space of graph structures. In fact, M$^3$HC finds slightly better solutions for only 4 of the 60 instances we tested and performed strictly worse than AGIP on the other 53 instances. Therefore, we can conclude that the exact methods DAGIP and AGIP are better at obtaining solutions with higher scores than the greedy method M$^3$HC on randomly generated instances. We also observe that AGIP improves over DAGIP 
	on 13 of 60 instances. There are two particular explanations for this. Firstly, the implied ground truth AG might be Markov equivalent to a DAG in which case the DAG solution can be optimal. Secondly, the ``non-DAG" candidate c-components are limited as AGIP only considers additionally c-components implied by a two-node district with up to 1 parent each node compared with DAGIP.

	\subsubsection{Comparison between Score-based Methods and Constraint-based Methods}
	\begin{table}[tb!]
		\centering
		\caption{Comparison between AGIP, M$^3$HC, FCI, and cFCI for ancestral ADMG Learning.}
		\begin{tabular}{ @{\extracolsep{\fill}} ccccccccccccc}
			\toprule
			$(d,l,N)$ & \multicolumn{4}{c}{SHD} & \multicolumn{4}{c}{Precision (\%)} & \multicolumn{4}{c}{Recall (\%)}\\
			\midrule\midrule\addlinespace
			& AGIP & M$^3$HC & FCI & cFCI & AGIP & M$^3$HC & FCI & cFCI & AGIP & M$^3$HC & FCI & cFCI\\
			\cmidrule(lr){2-5}\cmidrule(lr){6-9}\cmidrule(lr){10-13}
			$(18,2,\phantom{0}1000)$ & 36.0 & 32.0 & 32.9 & \textbf{25.9} &  41.7 & 37.6 & 30.9 & \textbf{49.5} & \textbf{41.9} & 30.8 & 24.7 & 38.5\\
			$(16,4,\phantom{0}1000)$ & \textbf{23.4} & 32.1 & 34.9 & 29.7 &  \textbf{58.0} & 35.9 & 27.6 & 41.0 & \textbf{56.6} & 29.3 & 22.8 & 32.5\\
			$(14,6,\phantom{0}1000)$ & \textbf{30.1} & 34.7 & 34.1 & 30.6 & \textbf{43.2} & 36.0 & 26.3 & 42.0 & \textbf{37.0} & 24.8 & 18.7 & 30.0\\
			$(18,2,10000)$ & \textbf{12.6} & 26.4 & 30.8 & 24.8 & \textbf{80.3} & 54.1 & 29.6 & 50.0 & \textbf{78.7} & 49.3 & 26.2 & 44.1\\
			$(16,4,10000)$ & 24.6 & 23.9 & 26.7 & \textbf{23.1} & \textbf{60.6} & 56.3 & 43.7 & 51.1 & \textbf{58.4} & 51.6 & 38.5 & 45.6\\
			$(14,6,10000)$ & 26.5 & 23.5 & 27.7 & \textbf{23.3} & 47.7 & \textbf{54.2} & 38.6 & 53.1 & 44.5 & \textbf{46.5} & 33.2 & 44.4\\
			overall & \textbf{25.5} & 28.8 & 31.2 & 26.2 & \textbf{55.3} & 45.7 & 32.8 & 47.8 & \textbf{52.8} & 38.7 & 27.4 & 39.2\\
			\bottomrule
		\end{tabular}
		\label{tab:constraint}
		\vskip 0.1in
	\end{table}
	\begin{figure}[tb!]
		\centering
		\begin{subfigure}{0.3\textwidth}
			\includegraphics[width=0.8\textwidth]{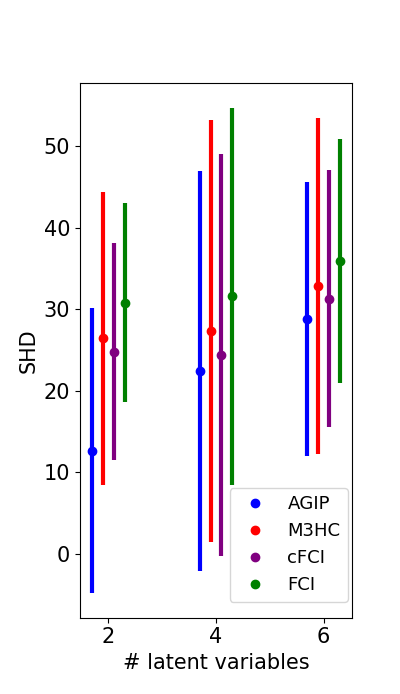}
		\end{subfigure}
		\begin{subfigure}{0.3\textwidth}
			\includegraphics[width=0.8\textwidth]{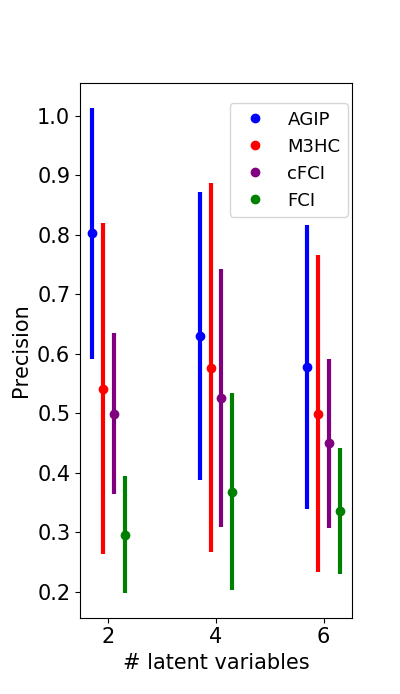}
		\end{subfigure}
		\begin{subfigure}{0.3\textwidth}
			\includegraphics[width=0.8\textwidth]{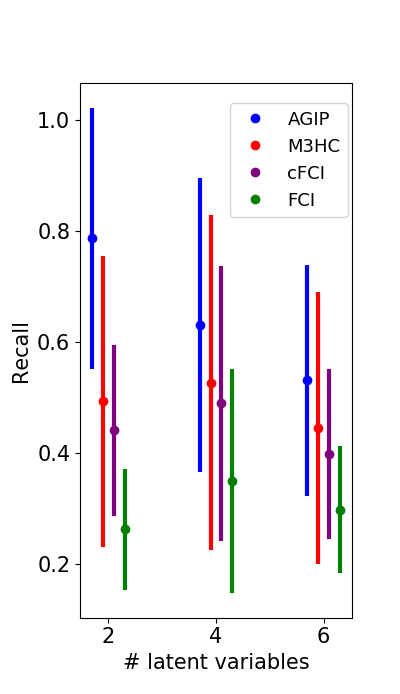}
		\end{subfigure}
		\caption{Performance of different methods when the number of latent variables increases.}
		\label{fig:num_latent}
	\end{figure}
	We compare score-based methods AGIP and M$^3$HC with constraint-based methods FCI and cFCI in Table \ref{tab:constraint}. 
	All methods have better precision and recall values when the sample size $N$ increases. FCI seems to perform worse than the other 3 methods, and on average, AGIP has the best SHD, precision, and recall.
	
	We observe that AGIP has a better recall in 5 out of 6 settings. 
	AGIP has significantly better performance than the other methods when $(d,l,N)=(18,2,10000)$.
	
	To see the impact of latent variables on the performance of these methods, we also regenerate instances with fixed $d=18$ and $N=10000$ but varying $l\in\{2,4,6\}$. For each $l$, we plot on Figure \ref{fig:num_latent} the mean ($\pm$ standard deviation) of SHD, precision, and recall of estimated graphs from each method over 10 instances (see the supplement for detailed results). The performance of each method tends to drop as the number of latent variables increases. The drop is most significant for AGIP, presumably due to the restriction on the c-component size. The ground truth graphs associated with our randomly generated instances can have c-components with large parental sizes and district sizes, especially in cases with more latent variables.
	
	\subsection{Experiment 3: Non-DAG ADMGs}
	\begin{table}[tb!]
		\centering
		\caption{Comparison between AGIP and DAGIP on graphs that are not DAG-representable.}
		\begin{tabular}{ @{\extracolsep{\fill}} ccccccccccccc}
			\toprule
			Graph index & \multicolumn{2}{c}{Avg SHD } & \multicolumn{2}{c}{Avg precision (\%)} & \multicolumn{2}{c}{Avg recall (\%)} & \# instances where AGIP\\
			& & & & & & & improves over DAGIP in score\\
			\midrule\midrule\addlinespace
			& AGIP & DAGIP & AGIP & DAGIP & AGIP & DAGIP\\
			\cmidrule(lr){2-3}\cmidrule(lr){4-5}\cmidrule(lr){6-7}
			1 & 6.7 & \textbf{6.6} & \textbf{63.7} & 59.5 & \textbf{64.4} & 60.0 & 10/10\\
			2 & \textbf{9.2} & 10.5 & \textbf{59.4} & 50.5 & \textbf{63.0} & 52.0 & \phantom{0}7/10\\
			3 & \textbf{8.0} & 8.8 & \textbf{67.3} & 64.8 & \textbf{63.8} & 60.0 & \phantom{0}5/10\\
			4 & \textbf{29.8} & \textbf{29.8} & 27.4 & \textbf{29.2} & 17.6 & \textbf{19.0} & \phantom{0}4/10\\
			5 & \textbf{21.7} & 23.0 & \textbf{30.0} & 27.6 & \textbf{27.3} & 24.7 & \phantom{0}2/10\\
			overall & \textbf{15.1} & 15.7 & \textbf{49.6} & 46.3 & \textbf{47.2} & 43.1 & \phantom{0}28/50\\
			\bottomrule
		\end{tabular}
		\label{tab:non-DAG}
		\vskip 0.1in
	\end{table}
	
	Although AGIP generates a solution with the same or better score than DAGIP, it returns a strictly better score than DAGIP on a small fraction (13/60) of instances in Experiment 2. One possibility is that all ground-truth graphs are DAGs.  
	We next compare AGIP and DAGIP on graphs that are not "DAG-representable", i.e., the ground truth AG is not  Markov-equivalent to any DAG. We randomly generate AGs with $(d,l)=(10,10)$ following the scheme described in Section \ref{exp:medium}. We pick the first 5 AGs that have at least one bidirected edge in the corresponding PAGs (which implies they are not DAG-representable). All of these 5 AGs are provided in the supplement. For each of the 5 AGs, we generate 10 instances of 10000 realizations of the model, each with a different randomly generated conditional linear Gaussian parametrization. In addition to the original candidate c-components, we consider in AGIP c-components that are implied by three-node districts with up to 1 parent for each node. Table \ref{tab:non-DAG} contains SHD, precision, and recall values of DAGIP and AGIP solutions.
	
	AGIP has a strictly better score than DAGIP in 28 out of the 50 instances we considered. Note that if the scores for optimal AGIP and DAGIP solutions are identical (as in 22 instances), then it is not possible to predict which solution will perform better on metrics such as SHD, precision, or recall. AGIP performs better than DAGIP in precision and recall on the 1st, 2nd, 3rd, and 5th AGs, and performs better in SHD on the 2nd, 3rd, and 5th AGs. AGIP performs slightly worse than DAGIP on the 4th AG, which has a very large c-component containing 7 nodes in the district. All of the five AGs contain at least one c-component that is not covered by AGIP (nor by DAGIP). But considering more c-components does help improve the solution quality on the average, which illustrates the advantage of AGIP.
	
	\section{Conclusions and Future Work}
	We presented an integer-programming based approach for learning ancestral ADMGs from observational data. Our main contributions are:  1) an IP formulation for the ancestral ADMG learning problem; 2) new classes of valid inequalities for efficient solution of the IP model; 3) numerical experiments to compare our model with existing methods. Empirical evaluation shows that our method has promising performance and can generate solutions with better accuracy than existing state-of-the-art learning methods.
	To our knowledge, this is the first exact score-based method for solving such problems in the presence of latent variables. For future work,  extending the current approach to allow efficient solution with more c-components could further improve the solution quality of the proposed method. Adding other classes of valid inequalities to strengthen the current IP formulation is another direction worth exploring.
	
	\section*{Acknowledgements}
	We thank Dr. James Luedtke for discussion on the problem formulation.
	
	\bibliography{ref}
	\bibliographystyle{ieeetr}
	
	\section*{Supplementary Material}
	\section*{Heuristics for Separation at Fractional Solutions}
	
	The starting point of our heuristic is Karger's random contraction algorithm for finding near-optimal min-cuts in edge-weighted undirected graphs (with nonnegative weights). Given a weighted graph with $n$ nodes and optimal min-cut value $t$ and a positive integer $\alpha \geq 1$, Karger's algorithms runs in time bounded by a polynomial function of $n^\alpha$ and returns all cuts in the graph with weight $ \leq \alpha t$. A weighted edge is chosen at random (with probability proportional to the weight of the edge), and the edge is contracted. When an edge $ij$ is contracted where $i$ and $j$ are (pseudo-)nodes, let $i'$ be a new pseudo-node representing $\{i,j\}$. Edges of the form $ki$ or $kj$ in the graph are removed and an edge $ki'$ with weight $w_{ki'} = w_{ki} + w_{kj}$ is added, where $w_{ki}$ is the weight of the edge $ki$ in the graph before contraction and 0 if no such edge exists. This contraction procedure is repeated till there are $2\alpha$ pseudo-nodes left, and the min-cut value in the resulting graph is returned. The central idea of the algorithm is that high weight edges are contracted resulting in the end-nodes of such edges being put in the same 'side' of the final cut.
	
	We adapt the above idea. We first discuss how to find violated strengthened cluster inequalities. Consider a subset $S \subseteq V$ and a solution vector $\bar z$ of the LP relaxation.  Let $\mu(S)$ equal the left-hand side of inequality where each $z$ variable is set to the corresponding value in $\bar z$. If we find a subset $S \subset V$ such that $\mu(S) < 1$, then we have found a cluster inequality violated by the point $\bar z$. However, as there are exponentially many choices of the set $S$, it is not realistic to enumerate each $S$ and compute $\mu(S)$. Instead, we initially only consider the sets $S = \{i\}$ consisting of individual nodes and note that $\mu(\{i\}) = 1$ for each node $i$ because of equation $\sum_{C:i\in D_C}z_{C}=1$. 
	Let $H_0$ be the undirected weighted graph with the same set of nodes as $G$.
	We iteratively select and contract ``high weight" edges and create pseudonodes (that consist of the union of nodes associated with the two pseudonodes incident to the edge), leading to a sequence of graphs $H_0, H_1, \ldots$, where each graph has one less pseudonode than the previous one. At the $k$th iteration we ensure that for each pseudonode $i \in H_k$, we have $\mu^k(\{i\})$ equal to the value of $\mu(S)$ where $S$ is the set of nodes in $H_0$ that correspond to the pseudonode $i$ of $H_k$.
	
	Let the weight of an edge $ij$ in $H_0$ be calculated as follows. Define \begin{multline*}
	w_{ij} := \sum_{W:j\in W}\sum_{C\in\mathcal{C}:i\in D_C,W_{C,i}=W}\bar{z}_C+
	\sum_{W:i\in W}\sum_{C\in\mathcal{C}:j\in D_C,W_{C,j}=W}\bar{z}_C+
	\sum_{C\in\mathcal{C}:\{i,j\}\subseteq D_C,i\notin W_{C,j},j\notin W_{C,i}}\bar{z}_C.
	\end{multline*}
	Note that the following relationship holds:
	\begin{equation}\label{wt-rel} \mu(\{i,j\}) = \mu(\{i\}) + \mu(\{j\}) - w_{ij}. \end{equation}
	Step 1: If we apply the random contraction step in Karger's algorithm to the weighted graph $H_0$ to obtain $H_1$, then with high probability we will contract an edge $ij$ with a high value of $w_{ij}$. This step leads to an $ij$ such that $\mu(\{i,j\})$ is approximately minimized (as $\mu(\{i\}) = \mu(\{j\}) = 1$ for all nodes $i,j$ of $H_0$).\\
	Step 2: We then create a pseudo-node $\{i,j\}$ in $H_1$ (labeled, say, by node $i$ if $i < j$ and by $j$ otherwise). Assuming the new psuedonode in $H_1$ has label $i$, We let $\mu^1(\{i\}) = \mu(\{i,j\})$ and $\mu^1(\{k\}) = \mu(\{k\})$ for all other nodes.\\ 
	Step 3: We then recalculate $w_{ij}$ values for edges in $H_1$ in such a fashion that for every pair of pseudonodes in $H_1$, the relationship in (\ref{wt-rel}) holds.
	To do this, we first remove all c-component variables $\bar z_C$ where $i \in D_C$ and $j \in W_{C,i}$ or $j \in D_C$ and $i \in W_{C,j}$. Next we replace all occurrences of $j$ by $i$ in the remaining variables, and then recompute the weights $w_{kl}$ for edges $kl$. 
	
	If we repeat Steps 1-3 for $H_1$ to obtain $H_2, H_3, \ldots$, then it is not hard to see that we always maintain the property in (\ref{wt-rel}) with $\mu$ replaced by $\mu^k$, and also the property that for any node $i$ in $H_k$, the value $\mu^k(\{i\})$ is equal to $\mu(S)$ where $S$ is the set associated with the pseudonode $i$.
	We stop whenever we find a pseudonode $i$ in $H_k$ (and associated $S$) such that $\mu^k(\{i\}) = \mu(S) < 1$.
	We repeat this algorithm multiple times while starting from different random seeds. Though this algorithm is not guaranteed to find a set $S$ such that $\mu(S) < 1$, it works well in practice, and does not return spurious sets $S$.
	
	To adapt the above algorithm to find violated strengthened bicluster inequalities, we proceed as follows. Consider a specific bidirected edge $ij$ such that $\bar I(i\leftrightarrow j) > 0$ for a given fractional solution $\bar z$.
	We first contract $ij$ in a special manner to obtain a graph $H_0$. Assume $i'$ represents the resulting pseudonode: for any c-component $C$ such that $i,j \in D_C$, we let $W_{C,i}$ and $W_{C,j}$ be replaced by a single parent set $W' = W_{C,i} \cup W_{C,j}$ of the new pseudonode $i'$. We also remove all c-component variables $z_C$ such that $D_C \cap \{i,j\} = 1$. We subsequently define $\mu(\{k\})$ values for nodes $k$ in $H_0$, edge weights $w_{kl}$, perform a random contraction step and repeat this process till we find a pseudonode $i$ in $H_k$ such that $\mu^k(\{i\}) < I(i \leftrightarrow j)$. We ensure that $\mu^k(\{i\})$ always represents the left-hand side of strengthened bicluster inequalities.
	
	\section*{Performance of different methods when the number of latent variables increases}
	We present in the following table the precise numbers (means of SHD, precision and recall) of the results in Figure 6 of the main paper.
	\begin{table}[htb!]
		\centering
		\caption{Exact numbers for Figure 6}
		\begin{tabular}{ @{\extracolsep{\fill}} ccccccccccccc}
			\toprule
			$l$ & \multicolumn{4}{c}{SHD} & \multicolumn{4}{c}{Precision (\%)} & \multicolumn{4}{c}{Recall (\%)}\\
			\midrule\midrule\addlinespace
			& AGIP & M$^3$HC & FCI & cFCI & AGIP & M$^3$HC & FCI & cFCI & AGIP & M$^3$HC & FCI & cFCI\\
			\cmidrule(lr){2-5}\cmidrule(lr){6-9}\cmidrule(lr){10-13}
			$2$ & \textbf{12.6} & 26.4 & 30.8 & 24.8 &  \textbf{80.3} & 54.1 & 29.6 & 50.0 & \textbf{78.7} & 49.3 & 26.2 & 44.1\\
			$4$ & \textbf{22.4} & 27.3 & 31.6 & 24.4 &  \textbf{63.0} & 57.7 & 36.9 & 52.6 & \textbf{63.1} & 52.7 & 34.9 & 48.9\\
			$6$ & \textbf{28.8} & 32.8 & 35.9 & 31.3 & \textbf{57.8} & 49.9 & 33.6 & 46.0 & \textbf{53.1} & 44.5 & 29.7 & 39.8\\
			\bottomrule
		\end{tabular}
		\label{tab:constraint}
		\vskip 0.1in
	\end{table}
	
	\section*{Ground Truth AGs for Experiments in Section 4.3 of the main paper}
	\begin{figure}[H]
		\centering
		\begin{subfigure}{0.3\textwidth}
			\begin{tikzpicture}
			\coordinate (Node1) at (0,-0);
			\coordinate (Node2) at (2*0.588,-2*0.191);
			\coordinate (Node3) at (2*0.951,-2*0.691);
			\coordinate (Node4) at (2*0.951,-2*1.309);
			\coordinate (Node5) at (2*0.588,-2*1.809);
			\coordinate (Node6) at (0,-2*2);
			\coordinate (Node7) at (-2*0.588,-2*1.809);
			\coordinate (Node8) at (-2*0.951,-2*1.309);
			\coordinate (Node9) at (-2*0.951,-2*0.691);
			\coordinate (Node10) at (-2*0.588,-2*0.191);
			\node[scale=1.5] at (Node1) {\textbf{1}};
			\node[scale=1.5] at (Node2) {\textbf{2}};
			\node[scale=1.5] at (Node3) {\textbf{3}};
			\node[scale=1.5] at (Node4) {\textbf{4}};
			\node[scale=1.5] at (Node5) {\textbf{5}};
			\node[scale=1.5] at (Node6) {\textbf{6}};
			\node[scale=1.5] at (Node7) {\textbf{7}};
			\node[scale=1.5] at (Node8) {\textbf{8}};
			\node[scale=1.5] at (Node9) {\textbf{9}};
			\node[scale=1.5] at (Node10) {\textbf{10}};
			\draw[<->,ultra thick,shorten <= 10pt,shorten >= 10pt] (Node7) -- (Node8);
			\draw[->,ultra thick,shorten <= 8pt,shorten >= 10pt] (Node1) -- (Node4);
			\draw[->,ultra thick,shorten <= 10pt,shorten >= 10pt] (Node2) -- (Node4);
			\draw[->,ultra thick,shorten <= 10pt,shorten >= 10pt] (Node2) -- (Node7);
			\draw[->,ultra thick,shorten <= 6pt,shorten >= 10pt] (Node3) -- (Node8);
			\draw[->,ultra thick,shorten <= 10pt,shorten >= 10pt] (Node10) -- (Node2);
			\draw[->,ultra thick,shorten <= 10pt,shorten >= 10pt] (Node10) -- (Node4);
			\draw[->,ultra thick,shorten <= 10pt,shorten >= 10pt] (Node10) -- (Node7);
			\draw[->,ultra thick,shorten <= 10pt,shorten >= 10pt] (Node10) -- (Node9);
			\end{tikzpicture}
			\caption*{AG \#1}
		\end{subfigure}
		\begin{subfigure}{0.3\textwidth}
			\begin{tikzpicture}
			\coordinate (Node1) at (0,-0);
			\coordinate (Node2) at (2*0.588,-2*0.191);
			\coordinate (Node3) at (2*0.951,-2*0.691);
			\coordinate (Node4) at (2*0.951,-2*1.309);
			\coordinate (Node5) at (2*0.588,-2*1.809);
			\coordinate (Node6) at (0,-2*2);
			\coordinate (Node7) at (-2*0.588,-2*1.809);
			\coordinate (Node8) at (-2*0.951,-2*1.309);
			\coordinate (Node9) at (-2*0.951,-2*0.691);
			\coordinate (Node10) at (-2*0.588,-2*0.191);
			\node[scale=1.5] at (Node1) {\textbf{1}};
			\node[scale=1.5] at (Node2) {\textbf{2}};
			\node[scale=1.5] at (Node3) {\textbf{3}};
			\node[scale=1.5] at (Node4) {\textbf{4}};
			\node[scale=1.5] at (Node5) {\textbf{5}};
			\node[scale=1.5] at (Node6) {\textbf{6}};
			\node[scale=1.5] at (Node7) {\textbf{7}};
			\node[scale=1.5] at (Node8) {\textbf{8}};
			\node[scale=1.5] at (Node9) {\textbf{9}};
			\node[scale=1.5] at (Node10) {\textbf{10}};
			\draw[<->,ultra thick,shorten <= 10pt,shorten >= 10pt] (Node4) -- (Node6);
			\draw[->,ultra thick,shorten <= 10pt,shorten >= 10pt] (Node1) -- (Node2);
			\draw[->,ultra thick,shorten <= 10pt,shorten >= 10pt] (Node1) -- (Node3);
			\draw[->,ultra thick,shorten <= 10pt,shorten >= 10pt] (Node1) -- (Node8);
			\draw[->,ultra thick,shorten <= 10pt,shorten >= 10pt] (Node1) -- (Node10);
			\draw[->,ultra thick,shorten <= 8pt,shorten >= 10pt] (Node2) -- (Node6);
			\draw[->,ultra thick,shorten <= 10pt,shorten >= 10pt] (Node7) -- (Node6);
			\draw[->,ultra thick,shorten <= 10pt,shorten >= 6pt] (Node8) -- (Node2);
			\draw[->,ultra thick,shorten <= 10pt,shorten >= 10pt] (Node8) -- (Node3);
			\draw[->,ultra thick,shorten <= 10pt,shorten >= 10pt] (Node8) -- (Node4);
			\end{tikzpicture}
			\caption*{AG \#2}
		\end{subfigure}
		\begin{subfigure}{0.3\textwidth}
			\begin{tikzpicture}
			\coordinate (Node1) at (0,-0);
			\coordinate (Node2) at (2*0.588,-2*0.191);
			\coordinate (Node3) at (2*0.951,-2*0.691);
			\coordinate (Node4) at (2*0.951,-2*1.309);
			\coordinate (Node5) at (2*0.588,-2*1.809);
			\coordinate (Node6) at (0,-2*2);
			\coordinate (Node7) at (-2*0.588,-2*1.809);
			\coordinate (Node8) at (-2*0.951,-2*1.309);
			\coordinate (Node9) at (-2*0.951,-2*0.691);
			\coordinate (Node10) at (-2*0.588,-2*0.191);
			\node[scale=1.5] at (Node1) {\textbf{1}};
			\node[scale=1.5] at (Node2) {\textbf{2}};
			\node[scale=1.5] at (Node3) {\textbf{3}};
			\node[scale=1.5] at (Node4) {\textbf{4}};
			\node[scale=1.5] at (Node5) {\textbf{5}};
			\node[scale=1.5] at (Node6) {\textbf{6}};
			\node[scale=1.5] at (Node7) {\textbf{7}};
			\node[scale=1.5] at (Node8) {\textbf{8}};
			\node[scale=1.5] at (Node9) {\textbf{9}};
			\node[scale=1.5] at (Node10) {\textbf{10}};
			\draw[<->,ultra thick,shorten <= 10pt,shorten >= 10pt] (Node2) -- (Node3);
			\draw[<->,ultra thick,shorten <= 10pt,shorten >= 10pt] (Node2) -- (Node10);
			\draw[<->,ultra thick,shorten <= 7pt,shorten >= 10pt] (Node3) -- (Node10);
			\draw[<->,ultra thick,shorten <= 10pt,shorten >= 10pt] (Node9) -- (Node10);
			\draw[->,ultra thick,shorten <= 10pt,shorten >= 10pt] (Node2) -- (Node4);
			\draw[->,ultra thick,shorten <= 10pt,shorten >= 10pt] (Node2) -- (Node9);
			\draw[->,ultra thick,shorten <= 10pt,shorten >= 10pt] (Node5) -- (Node6);
			\draw[->,ultra thick,shorten <= 10pt,shorten >= 10pt] (Node5) -- (Node9);
			\draw[->,ultra thick,shorten <= 10pt,shorten >= 10pt] (Node5) -- (Node10);
			\draw[->,ultra thick,shorten <= 10pt,shorten >= 6pt] (Node8) -- (Node3);
			\draw[->,ultra thick,shorten <= 10pt,shorten >= 10pt] (Node8) -- (Node6);
			\draw[->,ultra thick,shorten <= 10pt,shorten >= 10pt] (Node9) -- (Node4);
			\draw[->,ultra thick,shorten <= 10pt,shorten >= 10pt] (Node9) -- (Node6);
			\end{tikzpicture}
			\caption*{AG \#3}
		\end{subfigure}
	\end{figure}
	\begin{figure}[H]
		\centering
		\begin{subfigure}{0.3\textwidth}
			\begin{tikzpicture}
			\coordinate (Node1) at (0,-0);
			\coordinate (Node2) at (2*0.588,-2*0.191);
			\coordinate (Node3) at (2*0.951,-2*0.691);
			\coordinate (Node4) at (2*0.951,-2*1.309);
			\coordinate (Node5) at (2*0.588,-2*1.809);
			\coordinate (Node6) at (0,-2*2);
			\coordinate (Node7) at (-2*0.588,-2*1.809);
			\coordinate (Node8) at (-2*0.951,-2*1.309);
			\coordinate (Node9) at (-2*0.951,-2*0.691);
			\coordinate (Node10) at (-2*0.588,-2*0.191);
			\node[scale=1.5] at (Node1) {\textbf{1}};
			\node[scale=1.5] at (Node2) {\textbf{2}};
			\node[scale=1.5] at (Node3) {\textbf{3}};
			\node[scale=1.5] at (Node4) {\textbf{4}};
			\node[scale=1.5] at (Node5) {\textbf{5}};
			\node[scale=1.5] at (Node6) {\textbf{6}};
			\node[scale=1.5] at (Node7) {\textbf{7}};
			\node[scale=1.5] at (Node8) {\textbf{8}};
			\node[scale=1.5] at (Node9) {\textbf{9}};
			\node[scale=1.5] at (Node10) {\textbf{10}};
			\draw[<->,ultra thick,shorten <= 10pt,shorten >= 10pt] (Node1) -- (Node3);
			\draw[<->,ultra thick,shorten <= 10pt,shorten >= 10pt] (Node1) -- (Node4);
			\draw[<->,ultra thick,shorten <= 10pt,shorten >= 10pt] (Node1) -- (Node8);
			\draw[<->,ultra thick,shorten <= 10pt,shorten >= 10pt] (Node1) -- (Node9);
			\draw[<->,ultra thick,shorten <= 10pt,shorten >= 10pt] (Node1) -- (Node10);
			\draw[<->,ultra thick,shorten <= 10pt,shorten >= 10pt] (Node3) -- (Node4);
			\draw[<->,ultra thick,shorten <= 10pt,shorten >= 10pt] (Node3) -- (Node6);
			\draw[<->,ultra thick,shorten <= 10pt,shorten >= 10pt] (Node3) -- (Node8);
			\draw[<->,ultra thick,shorten <= 10pt,shorten >= 6pt] (Node3) -- (Node9);
			\draw[<->,ultra thick,shorten <= 10pt,shorten >= 10pt] (Node4) -- (Node6);
			\draw[<->,ultra thick,shorten <= 10pt,shorten >= 10pt] (Node6) -- (Node8);
			\draw[<->,ultra thick,shorten <= 10pt,shorten >= 6pt] (Node6) -- (Node9);
			\draw[<->,ultra thick,shorten <= 10pt,shorten >= 10pt] (Node8) -- (Node9);
			\draw[<->,ultra thick,shorten <= 10pt,shorten >= 10pt] (Node8) -- (Node10);
			\draw[->,ultra thick,shorten <= 10pt,shorten >= 10pt] (Node1) -- (Node6);
			\draw[->,ultra thick,shorten <= 10pt,shorten >= 10pt] (Node5) -- (Node3);
			\draw[->,ultra thick,shorten <= 10pt,shorten >= 10pt] (Node5) -- (Node4);
			\draw[->,ultra thick,shorten <= 8pt,shorten >= 6pt] (Node7) -- (Node2);
			\draw[->,ultra thick,shorten <= 10pt,shorten >= 10pt] (Node10) -- (Node3);
			\draw[->,ultra thick,shorten <= 10pt,shorten >= 10pt] (Node10) -- (Node6);
			\draw[->,ultra thick,shorten <= 10pt,shorten >= 10pt] (Node10) -- (Node9);
			\end{tikzpicture}
			\caption*{AG \#4}
		\end{subfigure}
		\begin{subfigure}{0.3\textwidth}
			\begin{tikzpicture}
			\coordinate (Node1) at (0,-0);
			\coordinate (Node2) at (2*0.588,-2*0.191);
			\coordinate (Node3) at (2*0.951,-2*0.691);
			\coordinate (Node4) at (2*0.951,-2*1.309);
			\coordinate (Node5) at (2*0.588,-2*1.809);
			\coordinate (Node6) at (0,-2*2);
			\coordinate (Node7) at (-2*0.588,-2*1.809);
			\coordinate (Node8) at (-2*0.951,-2*1.309);
			\coordinate (Node9) at (-2*0.951,-2*0.691);
			\coordinate (Node10) at (-2*0.588,-2*0.191);
			\node[scale=1.5] at (Node1) {\textbf{1}};
			\node[scale=1.5] at (Node2) {\textbf{2}};
			\node[scale=1.5] at (Node3) {\textbf{3}};
			\node[scale=1.5] at (Node4) {\textbf{4}};
			\node[scale=1.5] at (Node5) {\textbf{5}};
			\node[scale=1.5] at (Node6) {\textbf{6}};
			\node[scale=1.5] at (Node7) {\textbf{7}};
			\node[scale=1.5] at (Node8) {\textbf{8}};
			\node[scale=1.5] at (Node9) {\textbf{9}};
			\node[scale=1.5] at (Node10) {\textbf{10}};
			\draw[<->,ultra thick,shorten <= 10pt,shorten >= 10pt] (Node1) -- (Node7);
			\draw[<->,ultra thick,shorten <= 10pt,shorten >= 10pt] (Node1) -- (Node9);
			\draw[<->,ultra thick,shorten <= 10pt,shorten >= 10pt] (Node1) -- (Node10);
			\draw[<->,ultra thick,shorten <= 10pt,shorten >= 10pt] (Node7) -- (Node9);
			\draw[<->,ultra thick,shorten <= 10pt,shorten >= 10pt] (Node9) -- (Node10);
			\draw[->,ultra thick,shorten <= 8pt,shorten >= 10pt] (Node1) -- (Node5);
			\draw[->,ultra thick,shorten <= 10pt,shorten >= 10pt] (Node2) -- (Node1);
			\draw[->,ultra thick,shorten <= 10pt,shorten >= 10pt] (Node2) -- (Node4);
			\draw[->,ultra thick,shorten <= 10pt,shorten >= 10pt] (Node2) -- (Node7);
			\draw[->,ultra thick,shorten <= 10pt,shorten >= 10pt] (Node2) -- (Node9);
			\draw[->,ultra thick,shorten <= 10pt,shorten >= 10pt] (Node2) -- (Node10);
			\draw[->,ultra thick,shorten <= 10pt,shorten >= 10pt] (Node4) -- (Node1);
			\draw[->,ultra thick,shorten <= 10pt,shorten >= 8pt] (Node7) -- (Node10);
			\draw[->,ultra thick,shorten <= 6pt,shorten >= 6pt] (Node8) -- (Node3);
			\draw[->,ultra thick,shorten <= 10pt,shorten >= 10pt] (Node8) -- (Node9);
			\end{tikzpicture}
			\caption*{AG \#5}
		\end{subfigure}
	\end{figure}
	
\end{document}